%% file: neurips_2026.tex
\documentclass{article}

\PassOptionsToPackage{numbers,compress}{natbib}
 \usepackage[preprint]{neurips_2026}

\usepackage[utf8]{inputenc} 
\usepackage[T1]{fontenc}    
\usepackage{hyperref}       
\usepackage{url}            
\usepackage{booktabs}       
\usepackage{amsfonts}       
\usepackage{nicefrac}       
\usepackage{microtype}      
\usepackage{xcolor}         

\usepackage{amsmath}
\usepackage{amssymb}
\usepackage{mathtools}
\usepackage{amsthm}

\usepackage{bbding}
\usepackage{hyperref}
\usepackage{url}
\usepackage{wrapfig}
\usepackage{graphicx}
\usepackage{booktabs}
\usepackage{diagbox}
\usepackage{multirow}
\usepackage{subcaption}
\usepackage{caption}
\usepackage{makecell}

\usepackage{algorithm}
\usepackage{algorithmic}

\theoremstyle{plain}
\newtheorem{theorem}{Theorem}[section]

\theoremstyle{definition}
\newtheorem{definition}[theorem]{Definition}
\newtheorem{assumption}[theorem]{Assumption}
\theoremstyle{remark}

\newcommand{\ie}{\emph{i.e.,} }

\title{Time Series Causal Discovery via Context-Conditioned and Causality-Augmented Pretraining}

\author{
  Biao Ouyang, \enspace 
  Tengxue Zhang,  \enspace
  Zhihao Zhuang,  \enspace
  Yang Shu\thanks{Corresponding author.},  \enspace
  Chenjuan Guo,  \enspace
  Bin Yang \\
  East China Normal University \\
  \texttt{\{bouyang, txzhang, zhuangzhihao\}@stu.ecnu.edu.cn} \\
  \texttt{\{yshu, cjguo, byang\}@dase.ecnu.edu.cn}
}

\begin{document}

\maketitle

\input{src/abstract}

\input{src/intro}

\input{src/relatedwork}

\input{src/problem}

\input{src/method}

\input{src/experiments}

\input{src/conclusions}




\medskip
{
\small
\bibliographystyle{plainnat}
\bibliography{references}

}

\newpage

\appendix

\input{src/appendix}


\end{document}

%% file: src/abstract.tex
\begin{abstract}
Causal discovery from time series is critical for many real-world applications, such as tracing the root causes of anomalies. Existing approaches typically rely on dataset-specific optimization, making it difficult to transfer their causal discovery capabilities to new time series governed by diverse causal mechanisms. 
In this paper, we propose \textbf{PTCD}, a novel \textbf{P}retraining framework for \textbf{T}ime-series \textbf{C}ausal \textbf{D}iscovery, which improves cross-task generalization through context-conditioned modeling and transferable causal augmentation. 
To model complex temporal causal dependencies, PTCD employs a dual-scale iterative attention mechanism to capture window-level causal relationships, and a Gaussian mixture with a context-level routing mechanism to handle heterogeneous exogenous distributions. 
To further address distribution shifts across causal graphs, PTCD adopts a pretraining paradigm on synthetic datasets that integrates intervention-based learning and a causal mixup strategy, promoting stable causal discovery and stronger generalization. Extensive experiments on multiple real-world out-of-distribution (OOD) datasets demonstrate that PTCD excels in both causal discovery and root cause identification.
\end{abstract}

%% file: src/intro.tex
\section{Introduction}
Accurate causal discovery from observed time series is fundamental to many real-world applications. Causal analysis can reveal the drivers of stock price fluctuations~\citep{DBLP:conf/nips/stock} and the factors influencing river water levels~\citep{DBLP:conf/iclr/causalriver}. It also facilitates identifying the sources of system failures and tracing how these failures propagate~\citep{DBLP:conf/iclr/aerca,DBLP:conf/iclr/IDI}, thereby enabling actionable insights and more robust decision-making.

As shown in Figure~\ref{fig:intro_pretrain}, recent end-to-end causal discovery methods generalize poorly across different causal mechanisms and require costly per-dataset optimization, especially for time series with complex temporal dependencies. 
We therefore propose a pretraining paradigm that learns transferable causal discovery capabilities from diverse time series and causal graphs, enabling stronger generalization to new time series and supports accurate causal graph recovery in both zero-shot settings and with lightweight fine-tuning.
However, building a generalizable pretraining framework for time series causal discovery faces two main challenges.

\textbf{Challenge 1:} \textit{Inherent complexity of temporal causal dependencies}. 
As shown in Figure~\ref{fig:intro_window}, channel 3 at time $t$ is influenced by channel 1 at $t-2$, which in turn is influenced by channel 2 at $t-4$. This reveals an iterative, cascading dependency across time steps. 
Existing methods~\citep{DBLP:conf/iclr/cuts,DBLP:conf/iclr/aerca} inadequately capture such iterative cross-window interactions.
This suggests that reliable temporal causal discovery requires modeling dependencies beyond individual windows.
Furthermore, in structural causal models, observed time series (endogenous variables) are influenced not only by other endogenous variables but also by unobserved factors (exogenous variables). Since exogenous distributions vary across time series and causal structures, existing methods~\citep{DBLP:conf/aistats/Dynotears,DBLP:conf/iclr/aerca} assuming a fixed Gaussian prior lack the flexibility to model diverse causal mechanisms.

\begin{wrapfigure}[33]{r}{0.48\columnwidth}
    \centering

    \begin{subfigure}[b]{0.98\linewidth}
        \centering
        \includegraphics[width=\linewidth]{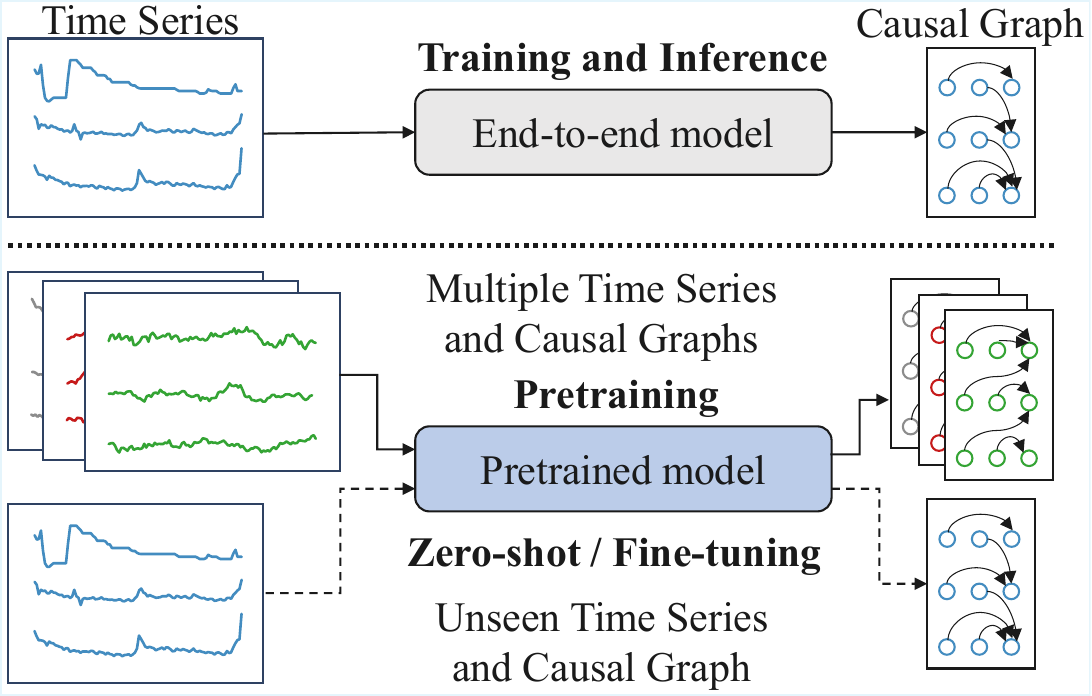}
        \caption{End-to-end method v.s. Pretraining method.}
        
        \label{fig:intro_pretrain}
    \end{subfigure}
    \hfill
    \begin{subfigure}[b]{0.98\linewidth}
        \centering
        \includegraphics[width=\linewidth]{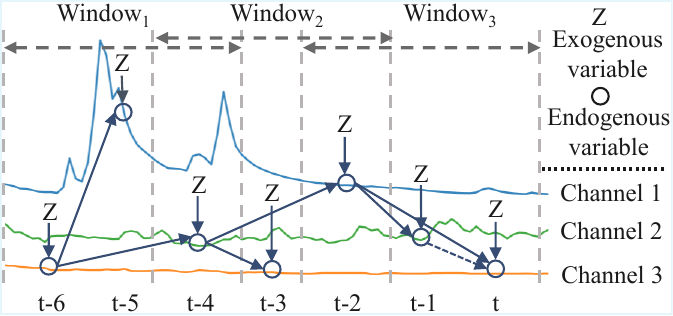} 
        \caption{Windowed causal dependencies.}
        \label{fig:intro_window}
    \end{subfigure}

    \caption{\textbf{(a)} End-to-end vs.\ pretraining paradigm: end-to-end models adopt dataset-specific optimization, while pretrained models learn from multiple time series and causal graphs, enabling zero-shot inference or fine-tuning on unseen time series/graphs. \textbf{(b)} Intra- and inter-window causal dependencies in multivariate time series $x$. Solid lines denote causal dependencies, while dashed lines indicate spurious correlations.}
    \label{fig:intro}
\end{wrapfigure}

\textbf{Challenge 2:} \textit{Lack of causality-aware generalization under distribution shifts}. 
Real-world time series often undergo distribution shifts caused by changing environments or system dynamics, leading to out-of-distribution (OOD) scenarios. Therefore, a robust causal discovery model should generalize beyond the training distribution.
However, existing methods fail to account for both the diversity of individual series and the variability of causal structures across different series.
Additionally, in causal discovery, distinguishing between correlation and causation is crucial~\citep{PearlMackenzie2018}. 
As shown in Figure~\ref{fig:intro_window}, both channel 3 at time $t$ ($x_{3,t}$) and channel 2 at $t-1$ ($x_{2,t-1}$) are influenced by channel 1 at $t-2$, which induces correlation between them. However, this does not imply a direct causal link from $x_{2,t-1}$ and $x_{3,t}$. What truly matters is determining whether a causal relationship exists. Unfortunately, existing deep learning methods struggle to transfer this causal reasoning capability to unseen time series and causal graphs.

For \textbf{Challenge 1,} we propose a \textit{hierarchical context-conditioned temporal causal discovery modeling approach}, which learns window-level representations at multiple scales and captures the distributions of exogenous variables in time series. Our framework features two key components: 
(1) \textit{Dual-scale iterative representation enhancement}, which employs an attention-based mechanism to iteratively model both intra-window and inter-window dependencies, capturing fine-grained causal dynamics within windows while accounting for coarse-grained long-range causal propagation across windows; and (2) \textit{Context-conditioned exogenous variable estimation}, which employs an adaptive mixture-of-Gaussians model with routing to approximate exogenous variable distributions in a context-conditioned manner, thereby enhancing the robustness of causal discovery across diverse time series.

For \textbf{Challenge 2,} we propose \textit{generalizable causal learning with causal augmentation} to strengthen the model's perception of causal relationships and improve its OOD robustness.
Specifically, the \textit{causal pretext task} introduces interventions into time series, breaking spurious correlations and providing strong signals for learning more generalizable causal structures. By treating the prediction of the intervened window as a pretext task, the model is encouraged to exploit discrepancies across varying time series and acquire transferable causal discovery capabilities.   
Furthermore, \textit{Time series causal mixup} operates by jointly mixing latent representations and their associated causal graphs, generating augmented sequences with diverse causal strengths and dependency patterns. This not only enriches the training data with diverse causal scenarios but also compels the model to generalize beyond previously observed correlation structures.  
In summary, our contributions are as follows:
\begin{itemize}
     \item We propose a novel pretraining paradigm for time-series causal discovery that can rapidly adapt to new time series and causal graphs.
    \item We introduce a hierarchical context-conditioned temporal causal discovery modeling that captures both inter-window and intra-window dependencies in individual time series, while flexibly handling context-conditioned exogenous noises across multiple series.
    \item We develop a generalizable causal pretraining strategy that employs intervention-based pretext tasks to break spurious correlations and reveal causal invariance in time series, alongside a causal mixup to achieve smoother gradients and improved robustness.
    \item Experiments on multiple real-world datasets show that PTCD consistently outperforms existing methods in both causal discovery and root cause identification.
\end{itemize}

%% file: src/relatedwork.tex
\section{Related Work}

\begin{wraptable}{r}{0.5\textwidth}  
  \centering
\vspace{-1em}
  
  \caption{Comparison of causal discovery methods. The symbol $\checkmark$ indicates that a component is incorporated, while $\times$ denotes its absence. Fine-grained and coarse-grained denote intra- and inter-window dependency, respectively. Exogenous refers to modeling of exogenous variables, and Generalization indicates the ability to handle unseen time series under different causal mechanisms.} 
\resizebox{0.48\textwidth}{!}{
\begin{tabular}{cccccc}
\toprule
Methods & Fine-grained & Coarse-grained & Exogenous & Generalization \\
\midrule
CDMI~\citep{CDMI} &      $\checkmark$        &        $\times$        &       $\times$         &         $\times$       \\
CUTS~\citep{DBLP:conf/iclr/cuts} &      $\checkmark$        &        $\times$        &       $\times$     &                  $\times$       \\
TCDF~\citep{DBLP:journals/make/tcdf} &      $\checkmark$        &        $\checkmark$        &       $\times$        &         $\times$       \\
CP~\citep{DBLP:journals/corr/cp}      &      $\times$        &        $\checkmark$        &       $\times$     &              $\checkmark$       \\
AERCA~\citep{DBLP:conf/iclr/aerca}   &       $\checkmark$       &         $\times$       &        $\checkmark$   &            $\times$          \\
\midrule
PTCD   &      $\checkmark$        &        $\checkmark$         &       $\checkmark$              &    $\checkmark$    \\
\bottomrule
\end{tabular}
}
\label{tab:relatedwork}
\end{wraptable}

\textbf{Causal Discovery for Time Series.} 
Granger causality-based methods~\citep{DBLP:journals/make/grange,DBLP:journals/make/grange2} assume that if past values of $X$ improve the prediction of future values of $Y$, then $X$ is a Granger cause of $Y$. 
This idea has been extended with various neural architectures: cLSTM~\citep{cMLP} leverages RNNs to infer Granger causal structures; TCDF~\citep{DBLP:journals/make/tcdf} utilizes attention-based CNNs for interpretable causal discovery;
GVAR~\citep{DBLP:conf/iclr/GVAR} adopts a vector autoregressive model with generalized coefficient matrices to increase modeling flexibility;
and CUTS~\citep{DBLP:conf/iclr/cuts,cuts+} learns a sparse causal adjacency matrix directly from data.
However, these methods share a key limitation: they do not explicitly model endogenous errors or exogenous noise, which restricts their applicability in practical scenarios. 
In contrast, Structural Causal Model (SCM) approaches explicitly characterize functional relationships among endogenous variables and account for exogenous influence: Varlingam~\citep{DBLP:journals/jmlr/Varlingam} combines a non-Gaussian instantaneous causality with vector autoregressive dynamics; 
TiMINo~\citep{DBLP:conf/nips/TiMINo} assumes that exogenous variables are independent over time; AERCA~\citep{DBLP:conf/iclr/aerca} uses an autoencoder to model both causal relationships and the distribution of exogenous variables. 
However, existing methods largely ignore distribution shifts, limiting their ability to generalize time series causal discovery across multiple datasets.
To address these gaps, we leverage a causal pretext task to uncover causal relationships in temporal sequences and causal mixup to generate diverse dependency patterns, thereby improving OOD generalization for causal discovery.

\textbf{Pre-trained Causal Discovery.}
Recent efforts have integrated causal discovery into pre-trained models~\citep{avici,DBLP:conf/iclr/CSIVA}. CaML~\citep{DBLP:conf/nips/CaML} formulates personalized effect prediction as a meta-task for zero-shot generalization; Cond-FiP~\citep{DBLP:journals/corr/Cond-FiP} dynamically infers structural causal models via a Fixed-Point Approach~\citep{DBLP:journals/corr/Fip}; and CInA~\citep{DBLP:conf/icml/CInA} learns transferable causal representations from unlabeled data, enabling zero-shot inference without fine-tuning. 
However, these methods do not explicitly capture temporal causality. CP~\citep{DBLP:journals/corr/cp} mitigates this by learning window-based causal graphs via supervised training across four architectures, yet it overlooks key temporal properties such as multi-scale dependencies and exogenous variable effects (see Table~\ref{tab:relatedwork}).
To overcome these limitations, we leverage dual-scale attention to capture inter-window and intra-window temporal dependencies, and context-conditioned exogenous estimation to model the distribution of exogenous variables in each time series.

%% file: src/problem.tex
\section{Problem Formulation}
We consider a multivariate time series $\boldsymbol{x} = \{x_{1:T, i}\}_{i=1}^C$ with  $C$ channels, where $x_{1:T, i}$ denotes the length-$T$ sequence for channel $i$. Following AERCA~\citep{DBLP:conf/iclr/aerca}, we adopt the Granger Causality framework and assume causal sufficiency, 
no instantaneous effects, and time-invariant causal relations. For structural identifiability, we assume that the system satisfies the causal minimality~\citep{DBLP:conf/icml/MCD} and follows a continuous nonlinear Additive Noise Model~\citep{ANM}. Under these assumptions, the data generation process is modeled as a time-invariant structural causal model (SCM):
\begin{equation}
\label{equ:def}
    x_{t,i} = g_i \left( \sum_{k=0}^{n-1} \sum_{j=1}^C G_{t-n+k,j,i} \cdot f_{k,j,i}(x_{t-n+k,j})\right) + Z_{i,t},
\end{equation} 
where $n$ is the maximum time lag, $\boldsymbol{G} \in \mathbb{R}^{n \times C \times C}$ is the causal graph, and $G_{t-n+k,j,i}$ indicates the impact of channel $j$ at time $t-n+k$ on $x_{t,i}$. The operator $\cdot$ denotes scalar multiplication, and the functions $g_i(\cdot)$ and $f_{k,j,i}(\cdot)$ are nonlinear transformations of the past observations, and $Z_{i,t}$ is an exogenous noise term for channel $i$ at time $t$. 

In this work, we focus on pretraining a causal discovery model $G=F_{\theta}(\boldsymbol{x})$ that acquires transferable causal discovery capability. The learned model is expected to generalize to unseen time series and enable the discovery of causal graphs that are not observed during pretraining. Moreover, we further transfer the pretrained model to real-world settings for causal discovery and root cause identification.

%% file: src/method.tex
\section{Methodology}
\begin{figure*}[t]
  \centering
  \includegraphics[width=0.95\linewidth]{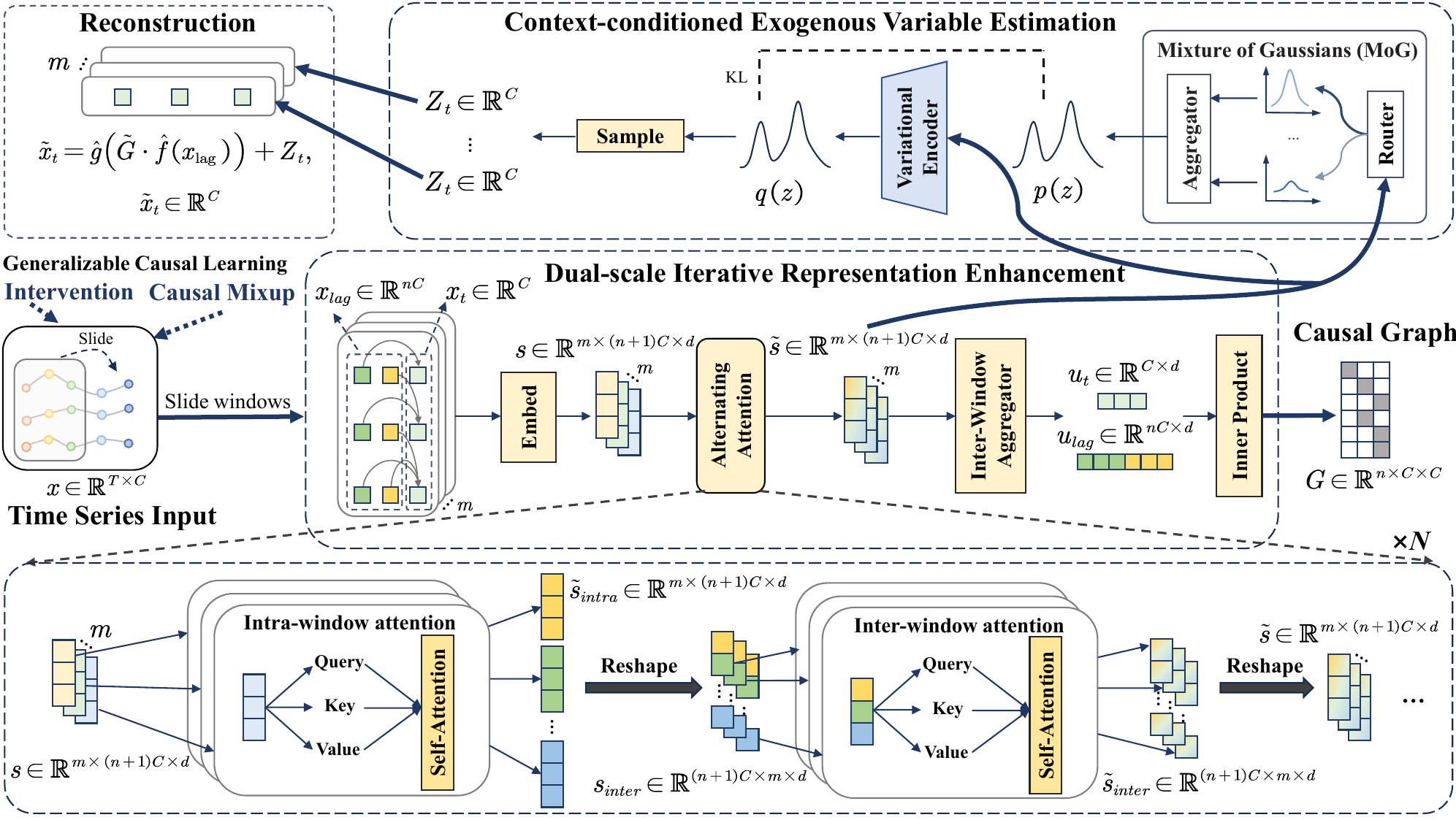}
  \caption{Overview of the hierarchical context-conditional temporal causal discovery framework, comprising: (1) Dual-scale Iterative Representation Enhancement, and (2) Context-conditioned exogenous variable estimation. The model's generalization is further improved via transferable causality-augmented pretraining, incorporating intervention pretext task and causal mixup strategies.
 }
 
  \label{fig:framework}
  \vspace{-1em}
\end{figure*}

We propose \textbf{PTCD}, a generalizable framework for time series causal discovery with context-conditioned modeling and causality-augmented pretraining (Figure~\ref{fig:framework}).
To capture complex temporal dependencies and diverse exogenous influences, we introduce a \textbf{hierarchical context-conditioned causal discovery framework}. 
The dual-scale iterative representation enhancement module segments time series into sliding windows and alternates intra-window and inter-window attention to refine temporal representations. An inter-window aggregator then produces past and current node embeddings, whose inner products estimate the window causal graph. To model heterogeneous exogenous variables, the context-conditioned exogenous variable estimation module adopts a routing-based mixture-of-Gaussians prior and reconstructs the current series from the estimated causal graph and past observations. 
To improve generalization, we introduce \textbf{transferable causal learning with causal augmentation}. The intervention pretext task performs interventions on the generative process of time series and formulates a prediction task for the intervened window, thereby eliminating spurious correlations and enabling stable causal discovery across different environments. causal mixup regularizes causal dependencies by jointly interpolating latent representations and causal graphs.
Finally, PTCD reconstructs time series by jointly modeling causal structure, endogenous dynamics, and exogenous factors, enabling interpretable and transferable causal discovery.

\subsection{Hierarchical Context-Conditional Temporal Causal Discovery Framework}

\textbf{Dual-scale Iterative Representation Enhancement.}
To model complex causal dependencies over both short and long time scales, as shown in Figure~\ref{fig:framework}, given a multivariate time series $x \in \mathbb{R}^{T \times C}$, we divide it into $m$ windows of size $(n+1)$, yielding $x_{window} \in \mathbb{R}^{m \times (n+1)C}$. To model the data generation process of the current time step $x_t$, we represent each window as the concatenation of past observations $x_{lag} \in \mathbb{R}^{nC}$ and the current step $x_t \in \mathbb{R}^{C}$. We then embed $x_{window}$ to obtain the window representation $s_{intra}\in \mathbb{R}^{m \times (n+1)C \times d}$, where $d$ represents the dimension of the embedding. Intra-window features capture local temporal dynamics, while inter-window relations provide global context. To jointly model these complementary scales, we propose an alternating attention mechanism that iteratively integrates features within and across windows.
First, we apply self-attention within each window. The sequence $s$ is projected into $Q^s_{intra}$, $K^s_{intra}$, and $V^s_{intra}$, and compute the intra-window representation as:
\begin{equation}
    \tilde{s}_{intra} = \text{Softmax}\!\left(Q^s_{intra}(K^s_{intra})^\top/{\sqrt{d_k}}\right) V^s_{intra}.
\end{equation}
Next, we perform inter-window self-attention to capture dependencies across windows. Specifically, the $\tilde{s}_{intra} \in \mathbb{R}^{m \times (n+1)C \times d}$ is reshaped as $s_{inter} \in \mathbb{R}^{(n+1)C \times m \times d}$, projected into $Q^s_{inter}$, $K^s_{inter}$, and $V^s_{inter}$, and processed as:
\begin{equation}
    \tilde{s}_{inter} = \text{Softmax}\!\left(Q^s_{inter}(K^s_{inter})^\top/{\sqrt{d_k}}\right) V^s_{inter}.
\end{equation}
The inter-window output $\tilde{s}_{inter}$ is reshaped back and fed into the next iteration. Repeating this alternating process for $N$ iterations progressively refines the representation, yielding the dual-scale enhanced representation $\tilde{s} \in \mathbb{R}^{m \times (n+1)C \times d}$. 
We then apply an inter-window aggregator over $\tilde{s}$ to obtain node features $u \in \mathbb{R}^{(n+1)C \times d}$ for window causal graph construction. Following the partition of $x_{lag}$ and $x_{t}$, we split $u$ into past nodes $u_{lag}$ and the current nodes $u_t$. 
Finally, we encode $u_{lag}$ and $u_t$ separately, and compute their inner product to estimate the link probability matrix $\tilde{G}$:
\begin{gather}
    u = Aggregator(\tilde{s}), \quad \tilde{G} =  \text{MLP}(u_{lag}) \text{MLP}(u_t)^T, \notag \\
    u \in \mathbb{R}^{(n+1)C \times d}, \quad u_{lag} \in  \mathbb{R}^{nC \times d}, \quad u_t \in  \mathbb{R}^{C \times d}.
\label{equ:uvg}
\end{gather}
Here, $\tilde{G}$ is a binary matrix indicating the presence or absence of causal links. The causal graph is trained to align the estimated link probability matrix $\tilde{G}$ with its ground-truth matrix $G$ as:
\begin{gather}
\label{eq:graph}
    \text{BCE}(G,\tilde{G}) = -G \cdot \log(\tilde{G}) - (1-G) \cdot \log(1-\tilde{G}).
\end{gather}

\textbf{Context-conditioned Exogenous Variable Estimation.}
Considering the diversity of time series and their underlying causal mechanisms, the distributions of exogenous variables can vary substantially across datasets. Therefore, we propose a context-conditioned approach that leverages the dual-scale enhanced representation $\tilde{s}$ as contextual input to estimate exogenous variable distributions. According to Equation (\ref{equ:def}), we define a data-generating distribution $p(D)$ that jointly encapsulates the underlying causal structure and the influence of exogenous factors:
\begin{equation}
\label{eq:generative_process}
x \sim p(D) = \int_{G} \int_{z} p(D \mid G, z) \cdot p(G) \cdot p(z) dzdG.
\end{equation}
Here, a sample $x$ is generated by first drawing from the prior distribution of exogenous variable $p(z)$ and the causal structure distribution $p(G)$ , and then iteratively generating data through the data-generating mechanism $p(D|G, z)$.

In time series, noise and external disturbances often exhibit diverse patterns, reflecting heterogeneous exogenous influences. Therefore, a single shared Gaussian prior is insufficient to model such variability.
To address this, we model the exogenous prior as a context-conditioned Gaussian mixture, where the mixing coefficients $\boldsymbol{\pi} \in \mathbb{R}^{K}$ are generated from the enhanced representation $\tilde{s}$ via a routing mechanism. The means $\mu$ and variances $\Sigma$ of the $K$ Gaussian components are learnable parameters:
\begin{equation}
    \boldsymbol{\pi} = \text{Router}(\tilde{s}), p(z) = \sum_{k=1}^K \pi_{k} \mathcal{N}(z \mid \mu_k, \Sigma_k), \sum_{k=1}^K \pi_{k} = 1,
\end{equation}
where $\text{Router}(\cdot)$ maps the context $\tilde{s}$ to mixing weights, and $\mathcal{N}$ denotes the Gaussian probability density function.

We employ a variational encoder to approximate the posterior distribution $q(z)$ of the exogenous variables conditioned on the temporal representation $\tilde{s}$. Using the reparameterization trick, we sample the exogenous variable $Z_t$ from $q(z)$ at each time step $t$. %
According to Eq.~\ref{equ:def}, the sampled variable $Z_t$ is then incorporated into the reconstruction of the current time series $x_t$, together with the predicted causal graph $\tilde{G}$, and the lagged endogenous variables $x_{lag}$. 
In this manner, both exogenous influences and the causal dependencies are explicitly modeled, ensuring that the learned representations are consistent with a structural causal model. The reconstruction of $x_t$ is formulated as:
\begin{equation}
\label{eq:recon}
    \tilde{x}_t = \tilde{g} \left ( \tilde{G}_p \cdot \tilde{f}(x_{lag}) \right ) + Z_t, 
\end{equation}
where $\tilde{g}(\cdot)$ and $\tilde{f}(\cdot)$ denote multi-layer perceptrons. We use the reconstruction error $\lVert x_t - \tilde{x}_t \rVert_2$ as the supervision signal for both pre-training and fine-tuning.

To regularize the latent space, we introduce a KL divergence term between the learned posterior distribution $q(z)$ and the assumed prior mixture Gaussian distribution $p(z)$:
\begin{equation}
     \mathcal{L}_{KL}(p(z) \parallel q(z)) = \int p(\mathbf{z}) \, 
\log \frac{p(\mathbf{z})}{q(\mathbf{z})} \, d\mathbf{z},\ 
\end{equation}

\subsection{Transferable Causal Learning with Causal Augmentation}

\textbf{Intervention Pretext Task.}
Time series data often exhibit autocorrelation and interdependencies, where past values influence the present, and multivariate series may involve cross-channel causality. Relying solely on observational data can confuse correlation and causality, making it difficult to distinguish between causal graphs within the same Markov equivalence class~\citep{environments,DBLP:conf/uai/interventions1}. 
Interventions on variables, by actively altering the underlying data-generating mechanisms, provide additional information that enhances causal identifiability.

We propose an intervention pretext task that uses interventions to facilitate causal structure learning. Specifically, we randomly select a time interval $[t_1, t_2]$ from the time series $x$ and replace the observation at $t_1$ with a randomly sampled value. Based on the underlying causal dynamics, we then iteratively propagate this intervention to compute the intervened segment $(x^{\prime}_{t_1}, \cdots, x^{\prime}_{t_2})$, replacing the original segment $(x_{t_1}, \cdots, x_{t_2})$ to obtain a modified series $x^{do}$. By examining segmented windows of $x^{do}$, we identify those inconsistent with the causal structure to derive binary ground-truth intervention labels.
To enhance generalization and encourage the model to capture stable causal dependencies, we train it to predict the window in which an intervention occurs.
We formulate intervention detection as a classification task for intervention detection:
\begin{equation}
\mathcal{L}_{do\_window} = \sum_{i=1}^m\left(-\, h_i \log(\hat{h}_i) - (1-h_i)\log(1-\hat{h}_i)\right),
\end{equation}
where $h_i$ and $\hat{h_i}$ denote the ground-truth and predicted intervention labels for window $i$, respectively.

Furthermore, Theorem~\ref{the:intervention} shows that optimizing the proposed intervention detection objective is equivalent to identifying the true causal mechanisms. The proof is provided in Appendix~\ref{thm:identifiability}.
\begin{theorem}
\label{the:intervention}
\textbf{(Causal Identifiability via Intervention Detection)} 
Let $P_{\text{obs}}(\mathbf{x})$ denote the observational distribution induced by the true SCM, and $P_{\text{int}}(\mathbf{x})$ denote the interventional distribution where a subset of variables $\mathbf{x}_{\mathcal{I}}$ (e.g., a time window) is replaced by samples from a noise or marginal proposal distribution $Q(\mathbf{x}_{\mathcal{I}})$. 
The optimal discriminator $D^*$ that minimizes the intervention classification loss is monotonically dependent on the true causal mechanism $P(\mathbf{x}_{\mathcal{I}} \mid Pa(\mathbf{x}_{\mathcal{I}}))$. 
Consequently, optimizing $\mathcal{L}_{do\_window}$ necessitates the model to learn the true underlying causal structure and the associated conditional dependencies.
\end{theorem}

To preserve semantic consistency under interventions, we apply contrastive learning to align the representation $u^{do}$ of the intervened series $x^{do}$ with the representation $u$ of the original $x$ using Eq.~\ref{equ:uvg}:

\begin{equation}
\begin{aligned}
\mathcal{L}_{do\_contrast} &= - \sum_{i=1}^{(n+1)C} \log\frac{\exp\!\left(\mathrm{sim}(u^{do,+},u_{i})/\tau\right)}
{\displaystyle\sum_{j=1}^{(n+1)C}\exp\!\left(\mathrm{sim}(u^{do}_{j},u_{i})/\tau\right)}, 
\end{aligned}
\end{equation}
where $\tau$ is a temperature parameter, $sim(\cdot,\cdot)$ is the similarity function, and $u^{do,+}$ denotes the positive pair of $u_i$ from the same node in the window causal graph. 
We jointly optimize the intervention detection objective and contrastive objective, \ie $\mathcal{L}_{do} = \mathcal{L}_{do\_window} + \mathcal{L}_{do\_contrast}$ to enhance causal discovery capability. The reconstruction loss on $x^{do}$ is computed following Eq.~\ref{eq:recon}.

\textbf{Time Series Causal Mixup.} 
To improve the generalization of the causal discovery framework, we propose a time series causal mixup strategy for temporal causal discovery. This operation smooths the solution space, yielding more stable and continuous estimates of edge existence. 
Specifically, we randomly sample two time series from the training dataset, each with latent embedding $\tilde{s}^i$ and its corresponding causal relation $G^i$. We then construct mixed samples as:
\begin{equation}
    \tilde{s}^m = \lambda_1\tilde{s}^1 + \lambda_2\tilde{s}^2, 
    \quad G^m = \lambda_1G^1 + \lambda_2G^2,
\end{equation}
where the mixing coefficient $\lambda_i$ is drawn from a Beta distribution.
By blending samples from diverse time series and varying causal graphs, time series causal mixup achieves a smoother optimization landscape during training, enhancing the transferability of causal dependency modeling.

We further present Theorem~\ref{thm:consistency}, which suggests that causal mixup enforces the learned causal probabilities to be transferable and physically consistent, avoiding the artifacts introduced by input-space interpolation. The proof is provided in Appendix~\ref{app:proofs_mixup}. Since function smoothness is closely related to generalization~\citep{DBLP:journals/tsp/smooth1,DBLP:conf/icml/smooth2}, PTCD achieves transferable causal discovery capability.

\begin{theorem}
\label{thm:consistency}
\textbf{(Geodesic Consistency of Causal Probability)} Under Assumption \ref{ass:linearity}, linear interpolation in the latent space $\mathcal{S}$ approximates a geodesic interpolation on the data manifold $\mathcal{M}$. 
Minimizing the Mixup empirical risk objective, which penalizes the discrepancy between the network prediction $\hat{G}(s^{mix})$ and the interpolation target $\lambda G(x_1) + (1 - \lambda)G(x_2)$ under a bounded loss $\epsilon$, inherently imposes an implicit gradient penalty proportional to the Frobenius norm of the Jacobian, $\mathcal{O}(\|\nabla_s \hat{G}(s)\|_F^2)$. This bounds the local Lipschitz constant of the causal mapping, regularizing the Jacobian field to be smooth along the transport path and ensuring stable causal graph discovery.
\end{theorem}

Following Eq.~\ref{eq:graph}, we compute the binary cross-entropy loss $\text{BCE}(G^m,\tilde{G}^{m})$ between the estimated link probability matrix $\tilde{G}$ and the mixup link probability matrix $G^m$.

\subsection{Causal Learning and Root Cause Identification}
\textbf{Pretraining Phase.} During pretraining, we generate synthetic datasets by iteratively simulating time series from causal graphs (see Appendix~\ref{apx:syn}) to train the model. 
Given a time series $x$, the objective function consists of several components: the reconstruction loss $\mathcal{L}_{recon}$, the causal graph loss $\mathcal{L}_{G}$, the KL divergence loss $L_{KL}$, the intervention loss $\mathcal{L}_{do}$. The overall pre-training loss is defined as:
\begin{equation}
    \mathcal{L}_{pre} = \mathcal{L}_{recon} + \lambda_G\mathcal{L}_G + \lambda_{KL}\mathcal{L}_{KL} + \lambda_{do}\mathcal{L}_{do},
\end{equation}
where $\lambda_G$, $\lambda_{KL}$, $\lambda_{do}$ are the corresponding weighting coefficients. The reconstruction loss $\mathcal{L}_{\mathrm{recon}}$ is computed over two types of input, the original series, the intervened series, and is defined as $\mathcal{L}_{recon} = \left\lVert x_{t} - \tilde{x}_{t} \right\rVert_{2} + \left\lVert x^{do}_{t} - \tilde{x}^{do}_{t} \right\rVert_{2}$. The causal graph loss $\mathcal{L}_{G}$ is computed over two types of input, the original series, the mixup series, and is defined as $\mathcal{L}_{G} = \text{BCE}(G,\tilde{G}) + \text{BCE}(G^m,\tilde{G}^m)$.

\textbf{Fine-tuning Phase.} During fine-tuning, we address two downstream tasks: \textit{Causal Discovery} and \textit{Root Cause Identification}. In real-world scenarios, the true data generation process is typically unavailable. Hence, we focus solely on the reconstruction loss with $s_t$,  
and the KL divergence loss $L_{KL}$. The overall fine-tuning loss is as follows:
\begin{equation}
\begin{aligned}
\mathcal{L}_{ft}
&= \lVert \tilde{s}_t - s_t \rVert_2
 + \lambda_{KL}\mathcal{L}_{KL}. 
\end{aligned}
\end{equation}

For \textit{Root Cause Identification}, following AERCA~\citep{DBLP:conf/iclr/aerca}, we score exogenous variable abnormality to pinpoint root causes.  
We compute the mean ($\mu$) and standard deviation ($\sigma$) of the exogenous distribution from normal data, and compute the root cause score as the z-score, $\mathrm{score}_t=\min_{k=1}^K\left(\frac{Z_t-\mu_k}{\sigma_k}\right)$. We then adopt Streaming Peaks-Over-Threshold~\citep{DBLP:conf/kdd/anomalyth} for threshold selection.

%% file: src/experiments.tex
\section{Experiments}

\subsection{Experimental Design}
\label{exp_design}
\textbf{Datasets.}
We construct a synthetic time series causal dataset (see Appendix~\ref{apx:syn}) for pretraining and fine-tune the model on two real-world tasks. 
\textit{Causal Discovery.} We evaluate on the German river benchmark~\citep{DBLP:conf/iclr/causalriver}, which includes causal graphs from Eastern Germany (666 stations) and Bavaria (494 stations), covering two variable types and five causal relation types: Close, Root Cause, Random+1, Confounder, and Random. The model is fine-tuned on Bavaria and evaluated on Eastern Germany. 
\textit{Root Cause Identification.} We fine-tune and evaluate on two benchmarks. The SWaT dataset~\citep{DBLP:conf/cpsweek/swat}, collected from a scaled-down water treatment testbed under both normal and cyber-attack conditions, is widely used for intrusion detection in industrial control systems. 
The MSDS dataset~\citep{DBLP:conf/esocc/msds}, generated on an OpenStack-based distributed infrastructure, injects controlled faults to emulate anomalies in multi-source cloud environments. Dataset details are provided in Appendix~\ref{apx:dataset}.

\textbf{Evaluation Metrics.}
\textit{Causal Discovery.} To avoid threshold tuning, we report the AUROC score~\citep{DBLP:journals/kais/metric1,DBLP:journals/tmlr/metric2,DBLP:conf/aistats/metric3,DBLP:journals/make/tcdf}.
\textit{Root Cause Identification.} Following prior work~\citep{DBLP:conf/nips/RCD,DBLP:conf/kdd/CIRCA,DBLP:conf/www/ack3,DBLP:conf/www/ack4,DBLP:conf/iclr/aerca}, we evaluate using recall at top-$k$, denoted as $Recall@K$. 
This metric measures the probability of correctly identifying root causes within the top-$k$ highest root cause scores. Further details are in Appendix~\ref{apx:details_metrics}.

\textbf{Baselines.}
\textit{Causal Discovery.}
We evaluate representative methods spanning classical and modern paradigms. Classical approaches include PCMCI~\citep{PCMCI}, Varlingam~\citep{DBLP:journals/jmlr/Varlingam}, Dynotears~\citep{DBLP:conf/aistats/Dynotears}, VAR~\citep{DBLP:conf/ijcai/VAR}, CDMI~\citep{CDMI}, TCDF~\citep{DBLP:journals/make/tcdf}, CUTS+~\citep{cuts+}. Among recent pretraining methods, we compare with Causal Pretraining (CP)~\citep{DBLP:journals/corr/cp}, which is implemented through both Transformer-based and GRU-based architectures. 
\textit{Root Cause Identification.} 
We compare PTCD with four baselines: 1) 
$\varepsilon$-Diagnosis~\citep{DBLP:conf/www/-Diagnosis}, 2) RCD~\citep{DBLP:conf/nips/RCD}, 3) CIRCA~\citep{DBLP:conf/kdd/CIRCA}, and 4) AERCA~\citep{DBLP:conf/iclr/aerca}.

\subsection{Experimental Results}

\begin{table*}[t]
\centering
\caption{Causal discovery results of AUROC on Eastern Germany river datasets, best in bold and second-best underlined.} 
\resizebox{0.98\textwidth}{!}{
\begin{tabular}{c|ccccccc|ccc}
\toprule
\diagbox{Dataset}{Model}   &  TCDF &  CUTS+    & VAR           & Varlingam     & Dynotears & PCMCI         & CDMI           & CP(Gru) & CP(Transformer)  & PTCD      \\
\midrule
Close (3)   & 0.57  &  0.56  & \underline{0.81} & 0.79 & 0.50       & 0.64          & \underline{0.81} & 0.79 & 0.75 & \textbf{0.84} \\
Close (5)   &  0.51   &  0.54  & 0.81 & 0.77          & 0.50       & 0.62 & 0.81          & 0.81 &  \underline{0.83} &   \textbf{0.85}       \\
Root cause (3)  &   0.55  &  0.56  & 0.79          & 0.77          & 0.56      & 0.70           & 0.75                   & 0.78  &  \underline{0.84}     &      \textbf{0.88}     \\
Root cause (5)   &  0.54   &  0.59  & 0.75          & 0.77 & 0.56      & 0.74          & 0.65           & 0.81      & \underline{0.84} &     \textbf{0.86}      \\
Random+1 (3)   &  0.72   &   0.56    & 0.80  & \underline{0.84} & 0.52      & 0.83          & 0.82            & 0.82       & 0.82          &     \textbf{0.87}      \\
Random+1 (5)   &  0.56  &    0.57   & 0.79     & 0.79          & 0.61      & 0.74          & 0.80      & \underline{0.84}     & \textbf{0.85} &    \textbf{0.85}       \\
Confounder (3) & 0.62 &  0.55  &  \underline{0.71} & 0.68          & 0.53      & 0.66          & 0.63                   & 0.64     & 0.65     &       \textbf{0.81}    \\
Confounder (5)  &  0.56 &  0.51  & \underline{0.72} & 0.70  & 0.53      & 0.64 & 0.71             & 0.71    & 0.71 &       \textbf{0.79}    \\
Random (3)   &   0.55  &  0.57  & \underline{0.82} & 0.79          & 0.50       & 0.65          & 0.80           &  0.77      & 0.81        &     \textbf{0.85}      \\
Random (5)   &  0.52  &  0.54  & 0.80  & 0.75          & 0.51      & 0.65          & 0.78      &   \underline{0.83}      & \textbf{0.86} &     \textbf{0.86}      \\
\midrule
\textbf{avg}   &  0.57  &  0.56  & 0.78          & 0.77          & 0.53      & 0.69          & 0.76      & 0.78          & 0.80          &    \textbf{0.85}     \\
\bottomrule
\end{tabular}
}
\label{tab:cd_r}
\vspace{-1em}
\end{table*}

\begin{table}[t]
\centering
\caption{Results of root cause identification, best in bold.}
\resizebox{0.91\columnwidth}{!}{
\begin{tabular}{c|c|ccccc}
\toprule
Dataset               & Model                & Recall@1                 & Recall@3        & Recall@5                 & Recall@10                & Avg@10               \\
\midrule
\multirow{5}{*}{MSDS} & $\epsilon$-Diagnosis & 0.004$\pm$0.004 & 0.266$\pm$0.002 & 0.452$\pm$0.009          & \textbf{1.000}$\pm$\textbf{0.000} & 0.492$\pm$0.001          \\
                      & RCD         & 0.412$\pm$0.048          & 0.573$\pm$0.010 & 0.984$\pm$0.001 & \textbf{1.000}$\pm$\textbf{0.000} & 0.821$\pm$0.012          \\
                      & CIRCA                & 0.454$\pm$0.238          & 0.860$\pm$0.140 & 0.917$\pm$0.084          & \textbf{1.000}$\pm$\textbf{0.000}          & 0.809$\pm$0.035          \\
                      & AERCA                & 0.381$\pm$0.408 & 0.908$\pm$0.062 & 0.974$\pm$0.027          & \textbf{1.000}$\pm$\textbf{0.000}          & 0.896$\pm$0.037          \\
                      & PTCD        &       \textbf{0.515}$\pm$\textbf{0.252}               &      \textbf{0.993}$\pm$\textbf{0.004}       &    \textbf{0.996}$\pm$\textbf{0.003}                 &             \textbf{1.000}$\pm$\textbf{0.000}         &     \textbf{0.929}$\pm$\textbf{0.013}      \\
\midrule
\multirow{5}{*}{SWaT} & $\epsilon$-Diagnosis          & 0.075$\pm$0.179          & 0.125$\pm$0.217 & 0.125$\pm$0.217          & 0.375$\pm$0.383 & 0.180$\pm$0.194          \\
                      & RCD         & 0.000$\pm$0.000          & 0.000$\pm$0.000 & 0.000$\pm$0.000          & 0.300$\pm$0.458          & 0.100$\pm$0.161          \\
                      & CIRCA       & 0.000$\pm$0.000 & 0.000$\pm$0.000 & 0.000$\pm$0.000 & 0.300$\pm$0.458          & 0.100$\pm$0.161 \\   
                      & AERCA       & 0.220$\pm$0.111          & 0.290$\pm$0.088 & 0.330$\pm$0.048          & 0.455$\pm$0.044          & 0.342$\pm$0.052          \\
                      & PTCD        &       \textbf{0.300}$\pm$\textbf{0.132}               &      \textbf{0.450}$\pm$\textbf{0.107}       &    \textbf{0.450}$\pm$\textbf{0.137}                  &             \textbf{0.475}$\pm$\textbf{0.141}         &   \textbf{0.440}$\pm$\textbf{0.105}  \\
\bottomrule
\end{tabular}
}
\label{tab:rc_r}
\end{table}

\begin{figure*}[t]
  \centering
  
  \begin{subfigure}[b]{0.19\textwidth}
    \centering
    \includegraphics[width=\linewidth]{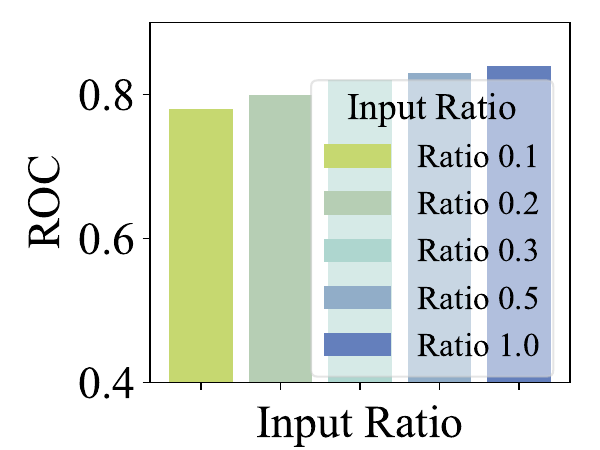}
    \caption{Close(3)}
  \end{subfigure}\hfill
  \begin{subfigure}[b]{0.19\textwidth}
    \centering
    \includegraphics[width=\linewidth]{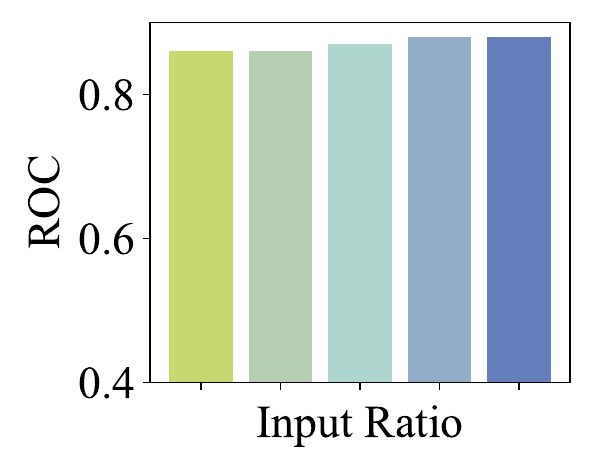}
    \caption{Root cause(3)}
  \end{subfigure}\hfill
  \begin{subfigure}[b]{0.19\textwidth}
    \centering
    \includegraphics[width=\linewidth]{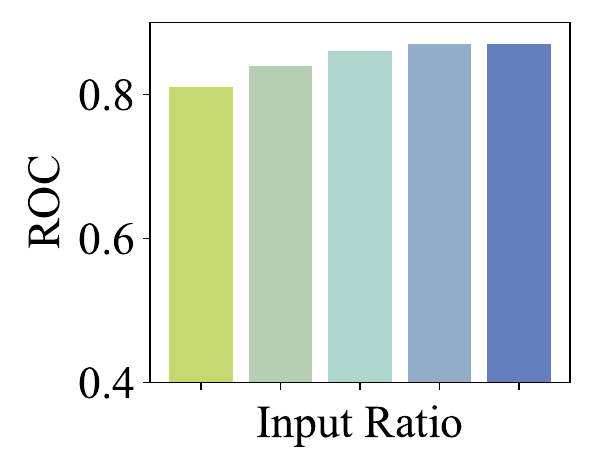}
    \caption{Random+1(3)}
  \end{subfigure}
  \begin{subfigure}[b]{0.19\textwidth}
    \centering
    \includegraphics[width=\linewidth]{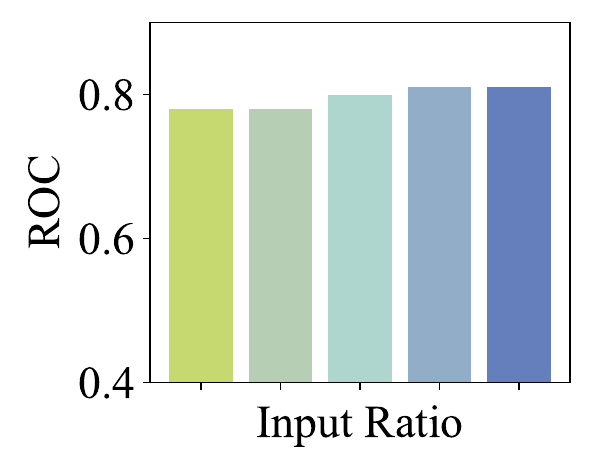}
    \caption{Confounder(3)}
  \end{subfigure}\hfill
  \begin{subfigure}[b]{0.19\textwidth}
    \centering
    \includegraphics[width=\linewidth]{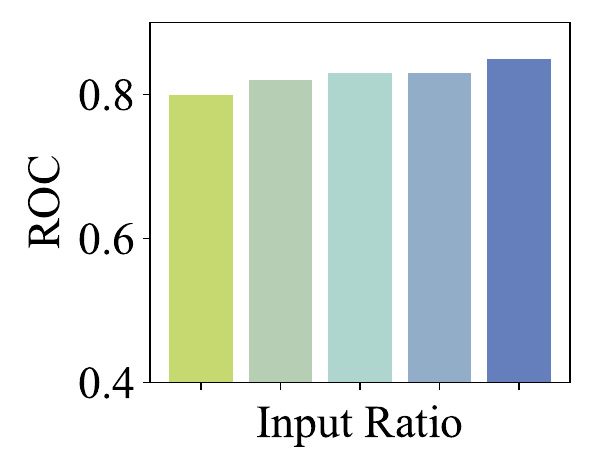}
    \caption{Random(3)}
  \end{subfigure}\hfill

  \vspace{0.6em} 

  \begin{subfigure}[b]{0.19\textwidth}
    \centering
    \includegraphics[width=\linewidth]{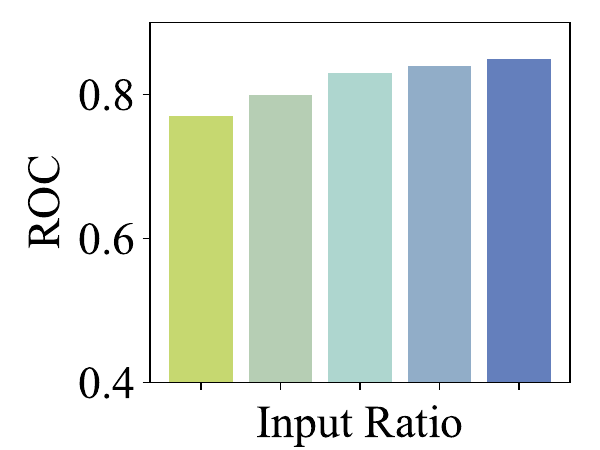}
    \caption{Close(5)}
  \end{subfigure}\hfill
  \begin{subfigure}[b]{0.19\textwidth}
    \centering
    \includegraphics[width=\linewidth]{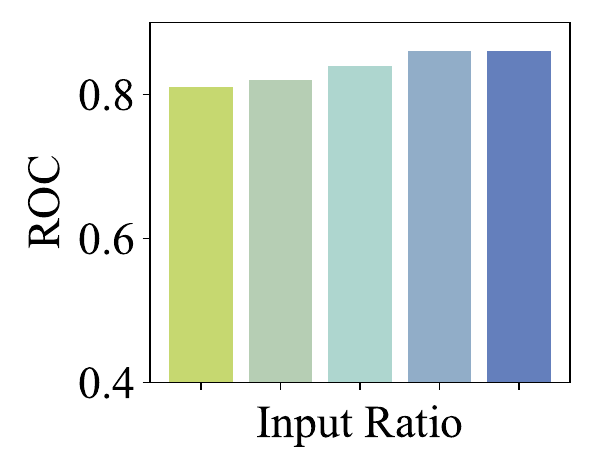}
    \caption{Root cause(5)}
  \end{subfigure}\hfill
  \begin{subfigure}[b]{0.19\textwidth}
    \centering
    \includegraphics[width=\linewidth]{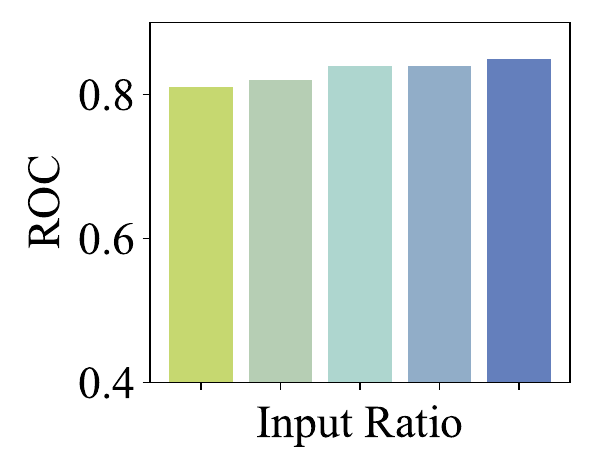}
    \caption{Random+1(5)}
  \end{subfigure}\hfill
  \begin{subfigure}[b]{0.19\textwidth}
    \centering
    \includegraphics[width=\linewidth]{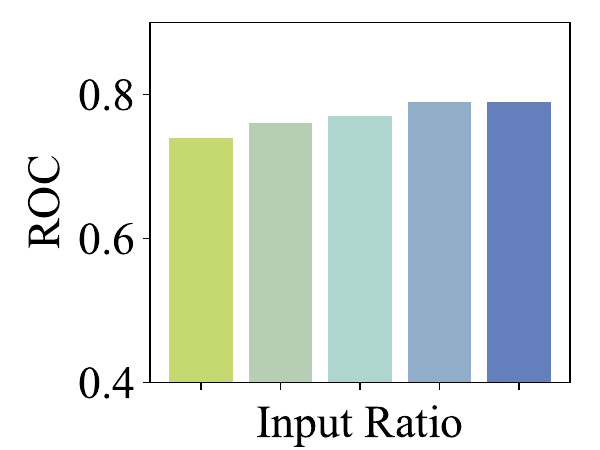}
    \caption{Confounder(5)}
  \end{subfigure}\hfill
  \begin{subfigure}[b]{0.19\textwidth}
    \centering
    \includegraphics[width=\linewidth]{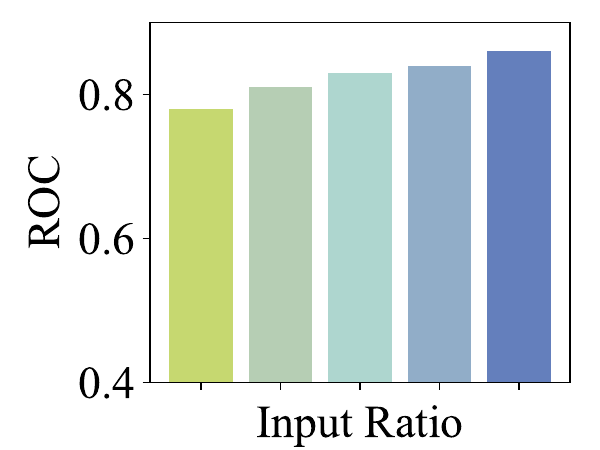}
    \caption{Random(5)}
  \end{subfigure}
  \caption{Results on the Eastern Germany river dataset (avg. length 3,143) under test input ratios of 10\%, 20\%, 30\%, 50\%, and 100\%.}
  \label{fig:input_radio}
\end{figure*}

\textit{Causal Discovery.} Table~\ref{tab:cd_r} presents the results for time series causal graph discovery, demonstrating that PTCD consistently outperforms all baselines across various types of causal relationships. 
This highlights the effectiveness of its dual-scale iterative representation enhancement and context-conditioned exogenous variables estimation in capturing complex causal dependencies. 
This approach enables the model to accurately identify true causal dependencies among time series while filtering out spurious correlations caused by confounders or environmental biases. For example, we observe that PTCD achieves strong causal discovery performance on the Eastern Germany river dataset, despite not being directly trained on it. Notably, compared to CP, our method demonstrates a significant improvement, with an average gain of 5\%. These results confirm that PTCD acquires a generalizable mechanism for causal discovery through intervention-based modeling and causal mixup, ensuring robustness even under real-world distribution shifts. 

\textit{Root Cause Identification.} Table~\ref{tab:rc_r} shows that PTCD achieves the best performance on real-world datasets in the root cause diagnosis task. 
Unlike AERCA, which primarily focuses on intra-window modeling, PTCD leverages a dual-scale iterative representation enhancement to capture both coarse-grained and fine-grained temporal patterns. This approach enhances the model’s ability to detect subtle variations over time, leading to superior performance in identifying the root cause of anomalies. Furthermore, the context-conditioned exogenous variable estimation boosts PTCD to adapt to the complex and often noisy conditions found in real-world data. 
According to the AC@1 metric, PTCD excels in accurately identifying the time series with the highest root cause score, demonstrating that incorporating intervention-based pre-training enables more stable detection of anomalous root causes even in the presence of dynamic shifts and external disturbances.

\textbf{Ablation Studies.}
\begin{table}[t]
\centering
\caption{Ablation study of causal discovery using the Eastern Germany river dataset. w/o inter, w/o intra, w/o exogenous, w/o causal mixup, w/o intervention represent removing the inter-window attention, intra-window attention, exogenous, causal mixup, and intervention, respectively.}

\resizebox{0.98\textwidth}{!}{
\begin{tabular}{c|cccccc}
\toprule

\diagbox{Dataset}{Model}       & w/o inter           & w/o intra     & w/o exogenous & w/o causal mixup         & w/o intervention          & PTCD  \\
\midrule
Close (3)      & 0.83 & 0.82 & 0.80       & 0.83          & 0.80 &   \textbf{0.84} \\
Close (5)      & 0.84 & 0.82          & 0.81       & 0.83 & 0.81 &    \textbf{0.85}       \\
Root cause (3) & 0.86          & 0.84          & 0.82      & 0.86           & 0.83                 &      \textbf{0.88}     \\
Root cause (5) & 0.85          & 0.84 & 0.82      & 0.85          & 0.83          &     \textbf{0.86}      \\
Random+1 (3)   & 0.85  & 0.84 & 0.83      & 0.86          & 0.84                 &     \textbf{0.87}      \\
Random+1 (5)   & 0.83          & 0.83          & 0.82      & 0.83          & 0.81  &    \textbf{0.85}       \\
Confounder (3) & 0.79 & 0.79          & 0.78      & 0.80          & 0.75          &       \textbf{0.81}    \\
Confounder (5) & 0.78 & 0.76  & 0.75      & 0.77 & 0.74        &       \textbf{0.79}    \\
Random (3)     & 0.84 & 0.84          & 0.81       & 0.83          & 0.82                  &     \textbf{0.85}      \\
Random (5)     & 0.85  & 0.84          & 0.83      & 0.84          & 0.83          &     \textbf{0.86}      \\
\midrule
\textbf{avg}  & 0.83          & 0.82          & 0.81      & 0.83          & 0.81                &    \textbf{0.85}     \\
\bottomrule
\end{tabular}
}
\label{tab:ablation_cd}
\end{table}
To assess the impact of different components in PTCD, we conduct an ablation study on inter-window attention, intra-window attention, mixture of Gaussians exogenous valuables, causal mixup, and intervention. As shown in Table~\ref{tab:ablation_cd}, each module contributes to the overall performance. Notably, removing exogenous variable estimation significantly reduces performance, highlighting the importance of learning context-conditioned distributions.
Furthermore, the inclusion of intervention yields notable performance improvements, particularly in datasets where the root cause is confounded (Confounder (3) and Confounder (5)). This demonstrates that interventions help the model distinguish between spurious correlations and true causal relationships, enhancing the accuracy of causal graph prediction. Further ablation studies are in Appendix~\ref{apx:ab_rc}.

\textbf{Different Time Series Ratio.} Figure~\ref{fig:input_radio} shows the impact of varying lengths of the input time series on performance. We find that using only 30\% of the full sequence length yields results comparable to the complete input, while even 10\% still provides reasonable results across all datasets. This shows that PTCD remains effective under limited resources and efficient for causal discovery using reduced temporal inputs, making it suitable for real-world applications.

%% file: src/conclusions.tex
\section{Conclusions}
This paper introduces PTCD, a novel pretraining framework for generalizable time series causal discovery across diverse downstream tasks. 
To capture complex temporal dependencies, PTCD employs a dual-scale attention mechanism for both fine-grained local dynamics and long-range causal relationships, complemented by a context-conditioned Gaussian mixture to accommodate heterogeneous exogenous variable distributions flexibly. 
To handle distribution shifts, PTCD incorporates a causal mixup and intervention-based learning to achieve causal augmentation, promoting stable causal discovery and stronger generalization. 
Experiments on real-world OOD datasets show that PTCD outperforms existing methods in causal discovery and root cause identification, establishing it as a robust foundation for time series analysis.

\clearpage

%% file: src/appendix.tex
\section{Proofs}
\label{app:proofs}

\subsection{Proofs of Intervention Pretext Task}
\label{thm:identifiability}
In this section, we provide the theoretical guarantee for the effectiveness of our proposed Intervention Pretext Task. We formally prove that minimizing the intervention detection loss is equivalent to estimating the true conditional causal mechanism, thereby enabling the model to identify the true parents in the causal graph.

\begin{proof}
The Intervention Pretext Task is formulated as a binary classification problem. Let $y=1$ denote samples from the observational distribution $P_{\text{obs}}$, and $y=0$ denote samples from the interventional distribution $P_{\text{int}}$. The objective is to minimize the binary cross-entropy loss $\mathcal{L}_{do}$:
\begin{equation}
    \mathcal{L}_{do} = - \mathbb{E}_{\mathbf{x} \sim P_{\text{obs}}} [\log D(\mathbf{x})] - \mathbb{E}_{\mathbf{x} \sim P_{\text{int}}} [\log (1 - D(\mathbf{x}))],
\end{equation}
where $D(\mathbf{x})$ represents the probability that $\mathbf{x}$ is from the observational distribution.

According to the classical result in density ratio estimation~\citep{gan}, for fixed distributions $P_{\text{obs}}$ and $P_{\text{int}}$, the optimal discriminator $D^*(\mathbf{x})$ that minimizes this loss is given by:
\begin{equation}
    D^*(\mathbf{x}) = \frac{P_{\text{obs}}(\mathbf{x})}{P_{\text{obs}}(\mathbf{x}) + P_{\text{int}}(\mathbf{x})}.
    \label{eq:optimal_D}
\end{equation}

Now, consider the factorization of the joint probability based on the causal graph $G$. Let the intervention be performed on a specific variable (in window) $x_i$ at time $t$, replacing it with a value drawn from $Q(x_i)$.
The observational distribution factors according to the causal Markov condition:
\begin{equation}
    P_{\text{obs}}(\mathbf{x}) = P(x_i \mid Pa(x_i)) \cdot \prod_{j \neq i} P(x_j \mid Pa(x_j)).
\end{equation}
The interventional distribution, where the causal mechanism of $x_i$ is replaced by $Q(x_i)$ while other mechanisms remain invariant, factors as:
\begin{equation}
    P_{\text{int}}(\mathbf{x}) = Q(x_i) \cdot \prod_{j \neq i} P(x_j \mid Pa(x_j)).
\end{equation}

Substituting these factorizations into Eq. (\ref{eq:optimal_D}), the terms for all non-intervened variables $\prod_{j \neq i} P(x_j \mid Pa(x_j))$ are common to both the numerator and the denominator and thus cancel out. The optimal discriminator simplifies to:
\begin{equation}
    D^*(\mathbf{x}) = \frac{P(x_i \mid Pa(x_i))}{P(x_i \mid Pa(x_i)) + Q(x_i)} = \sigma \left( \log \frac{P(x_i \mid Pa(x_i))}{Q(x_i)} \right),
\end{equation}
where $\sigma(\cdot)$ is the sigmoid function.

Since $Q(x_i)$ is a known proposal distribution (e.g., mixture Gaussian distribution), and $\sigma(\cdot)$ is a strictly monotonic function, the output of the optimal discriminator $D^*(\mathbf{x})$ is determined solely by the true conditional probability $P(x_i \mid Pa(x_i))$.

Therefore, to minimize the loss $\mathcal{L}_{do}$, the model $D$ is forced to approximate $D^*$. This implies that the model must accurately estimate the conditional likelihood $P(x_i \mid Pa(x_i))$. An accurate estimation of this term is only possible if the model correctly identifies the true parent set $Pa(x_i)$ and models the correct functional relationship. Relying on spurious correlations would yield a suboptimal estimate of the conditional density, resulting in a higher loss. This proves that the intervention pretext task provides a sufficient signal for recovering the causal parents and discovering the underlying causal relationships.
\end{proof}

\subsection{Proofs of Time Series Causal Mixup}
\label{app:proofs_mixup}
Let $x_t \in \mathcal{X} \subseteq \mathbb{R}^d$ denote the observed system state. The dynamics are governed by a structural equation $x_{t+1} = f^*(x_t) + \epsilon$, where $f^*$ is the ground-truth mechanism. We focus on estimating the \textbf{Causal Probability Matrix}, defined as the Jacobian $G(x) := J_{f^*}(x) = \frac{\partial f^*}{\partial x} \in \mathbb{R}^{d \times d}$.

Standard mixup operates in the input space $\mathcal{X}$. However, physical observations usually lie on a low-dimensional manifold $\mathcal{M} \subset \mathcal{X}$ embedded in the high-dimensional ambient space. A convex combination $\lambda x_1 + (1-\lambda)x_2$ often falls off the manifold $\mathcal{M}$, leading to physically invalid states where the causal mechanism $f^*$ is undefined.

To resolve the geometric inconsistency, we leverage the concept of causal representation learning.

\begin{definition}[Latent Causal Embedding]
\label{def:embedding}
We assume there exists a latent space $\mathcal{S} \subseteq \mathbb{R}^k$ and an encoder function $\phi: \mathcal{X} \to \mathcal{S}$. While $\phi$ is implemented as a neural network, we assume it acts as a smooth embedding on the support of the data manifold $\mathcal{M}$. Specifically, $\phi|_{\mathcal{M}}$ is injective and preserves the topological structure of the causal mechanisms.
\end{definition}

This definition aligns with the \textit{Manifold Hypothesis}, where deep networks learn to unfold complex curved manifolds into flatter representations \cite{bengio2013representation}.

\begin{assumption}[Linearity of Mechanisms in Disentangled Space]
\label{ass:linearity}
Following the principle of Independent Causal Mechanisms (ICM) \cite{scholkopf2021toward}, we assume that in the ideal disentangled latent space $\mathcal{S}$, the dynamics $g(s)$ (where $s=\phi(x)$) are governed by simpler, smoother mechanisms compared to the ambient space. Locally, the evolution in $\mathcal{S}$ can be well-approximated by linear transitions.
\end{assumption}

\textbf{Geodesic Consistency via Causal Mixup.}
We propose performing mixup in $\mathcal{S}$. Let $s_1 = \phi(x_1)$ and $s_2 = \phi(x_2)$. The mixed state is $s^{mix} = \lambda z_1 + (1-\lambda)z_2$.

\textbf{1. Geodesic Interpretation.}
Since $\phi$ acts as a smooth embedding (Definition \ref{def:embedding}), the image of the data manifold $\phi(\mathcal{M})$ in $\mathcal{S}$ is locally homeomorphic to Euclidean space. A straight line segment in $\mathcal{S}$ defined by $\gamma_s(\lambda) = \lambda s_1 + (1-\lambda)s_2$ corresponds to a curve $\gamma_x(\lambda) = \phi^\dagger(\gamma_z(\lambda))$ in the original space $\mathcal{X}$, where $\phi^\dagger$ denotes the pseudo-inverse or decoder mapping.
Because $\phi$ unfolds the manifold, the curve $\gamma_x(\lambda)$ approximates the \textbf{geodesic} (shortest path) between $x_1$ and $x_2$ constrained to lie on $\mathcal{M}$. Unlike input mixup, which creates off-manifold samples, Causal Mixup generates valid physical states along this geodesic.

\textbf{2. Jacobian Dynamics and Chain Rule.}
Let the dynamics in the latent space be described by $s_{t+1} = g(s_t)$. The relationship between the observed causal probability $G(x)$ and the latent mechanism is given by the chain rule:
\begin{equation}
    G(x) = \frac{\partial f^*}{\partial x} = J_{\phi^\dagger}(s) \cdot \frac{\partial g}{\partial s} \cdot J_{\phi}(x)
\end{equation}
where $J_{\phi}$ and $J_{\phi^\dagger}$ are the Jacobians of the encoder and decoder, respectively.

\textbf{3. Regularization via Mixup and Jacobian Smoothing.} Causal Mixup optimizes the model by minimizing the mixup loss between the network's predicted causal probability and the linear interpolation target. The empirical risk objective can be formally defined as:
\begin{equation}
    \mathcal{L}_{mixup} = \mathbb{E}_{\lambda} \left[ \left\| \hat{G}(s^{mix}) - \left( \lambda G(x_1) + (1 - \lambda)G(x_2) \right) \right\|_F^2 \right] \quad (21)
\end{equation}
where $s^{mix} = s_2 + \lambda(s_1 - s_2)$ is the mixed latent state, $\hat{G}(\cdot)$ is the neural network predictor, and $G(\cdot)$ represents the target causal mechanism.

To rigorously quantify the implicit regularization effect of this objective, we employ Taylor's Theorem with the Lagrange remainder. By defining the perturbation direction as $\Delta s = s_1 - s_2$, the mixed state can be rewritten as $s^{mix} = s_2 + \lambda \Delta s$. Applying a first-order Taylor expansion to the network output yields:
\begin{equation}
    \hat{G}(s_2 + \lambda \Delta s) = \hat{G}(s_2) + \lambda \nabla_s \hat{G}(s_2) \Delta s + R_2
\end{equation}
where $R_2$ represents the Lagrange remainder. Minimizing the objective $\mathcal{L}_{mixup}$ forces the Taylor expansion of the network output to match the linear interpolation of the targets. This optimization process yields an implicit gradient penalty proportional to the Frobenius norm of the Jacobian, denoted as $\mathcal{O}(\|\nabla_s \hat{G}(s)\|_F^2)$, with the approximation error tightly controlled by the spectral norm of the Hessian.

Formally, if the mixup loss is bounded by a small value $\epsilon$, the Hessian norm is bounded by $\mathcal{O}(\epsilon)$. This implies that the Jacobian $\nabla_s \hat{G}(s)$ is locally constant up to an $\mathcal{O}(\epsilon)$ variation along the interpolation path. Consequently, minimizing the mixup loss inherently bounds the local Lipschitz constant of the causal mapping $\hat{G}$. Under Assumption~\ref{ass:linearity}, this mathematical equivalence ensures the global smoothness of the causal field and effectively prevents structural instability in the inferred causal graphs.

\textbf{4. Handling ReLU Networks (Piecewise Linearity).}
While we assumed smoothness, modern neural networks (MLPs with ReLU) are piecewise linear functions. The input space is partitioned into polytopes where the network is an affine transformation.
Within each polytope, $\phi$ is strictly linear, and its Jacobian $J_\phi$ is constant. Causal Mixup interpolates across these regions. By enforcing the linearity of the output $G$ with respect to the convex combination in $\mathcal{S}$, we effectively regularize the transitions across the boundaries of these linear regions, ensuring that the global causal field remains continuous and smooth.

Causal Mixup acts as a regularizer that enforces the \textit{Parallel Transport} of the causal mechanism along the data manifold. This minimizes the curvature of the learned causal field, leading to better generalization and structural identifiability compared to input-space interpolation. Therefore, time series causal Mixup not only enables a valid mixup process but also boosts the model's generalization performance.

\section{Experiments}

\subsection{Details of Datasets}
\label{apx:dataset}

\textbf{Synthetic Datasets.}
\label{apx:syn}

Referring the data generation process of RHINO~\citep{DBLP:conf/iclr/Rhino}, we generate diverse time series and causal graphs. For the synthetic data pipeline, we construct the synthetic datasets using a four-step simulation process designed to capture diverse causal mechanisms. The procedure is defined as follows:

\begin{enumerate}
    \item \textbf{Topology Generation:} We generate random Erd\H{o}s--R\'enyi (ER) or Scale-Free (SF) graphs~\citep{DBLP:journals/make/er,DBLP:journals/make/sf} to specify the underlying lagged causal relationships between variables.
    \item \textbf{Exogenous Noise Sampling:} For each time series, the exogenous variables were drawn from random Gaussian Mixture Models (GMMs) to introduce non-Gaussian complexity.
    \item \textbf{Temporal Simulation:} We sample initial starting conditions and simulate the temporal progression of the system following Eq.~\ref{equ:def} with additive noise. The functional relationships governing the dynamics were modeled using MLPs with two hidden layers (64 units each) and ReLU activation functions.
    \item \textbf{Stabilization:} A burn-in period was simulated and subsequently discarded to ensure the time series reached a stable, stationary regime.
\end{enumerate}

For dataset specifications, we generate datasets across multiple axes of variation. The specific configurations are detailed below:

\begin{itemize}
    \item \textbf{System Size and Topology:} The number of nodes varies from 5 to 40. Both ER and SF graph structures are utilized.
    \item \textbf{Sequence Length:} Each generated time series has a length of $T=100$ steps, obtained after removing an initial burn-in period of 100 steps.
    \item \textbf{Training Dataset:} The training corpus is designed to cover a wide range of structural variations. We generate 1,000 distinct causal graphs for each node count. For each causal graph, we simulate 5 independent multivariate time series. This yields a total dataset size of 180,000 time series.
    \item \textbf{Validation and Test Datasets:} The validation and test datasets each contain 370 time series, along with their corresponding ground-truth causal graphs for evaluation. Notably, these causal graphs are strictly disjoint from the training dataset.
    \item \textbf{Interventional Data:} To facilitate the learning of the intervention pretext task, we generate interventional samples for the training dataset. Specifically, we first sample a start time $t_1$ and an end time $t_2$. We then randomly resample the specific treatment variables at time $t_1$. Conditioned on this intervention, we iteratively generate the sequence over the interval $(t_1,t_2]$. Finally, we assign a binary label to each window indicating whether it contains an intervention.
\end{itemize}

\textbf{Real-world Benchmark.} 
\label{apx:real_data}
CausalRiver~\citep{DBLP:conf/iclr/causalriver} treats river flow measurements from monitoring stations as a time series dataset with inherent causal structure. It covers two regions in Germany: the eastern German territory (666 measuring stations) and Bavaria (494 stations). The dataset spans from 2019 to 2023 with a temporal resolution of 15 minutes. Importantly, CausalRiver includes both normal hydrological conditions and extreme events such as heavy rainfall and large-scale precipitation, enabling the study of causal dynamics under diverse environmental scenarios. The dataset is categorized into five types based on the characteristics of their underlying causal structures:
\begin{itemize}
    \item Random: All connected subgraphs with three or five nodes, covering the entire dataset and full diversity of benchmark conditions.
    \item Close: Connected subgraphs whose edges have a maximum geographic (Euclidean) distance of 0.3; by excluding long-range connections, causal effects are expected to be more pronounced. This set is fully contained within Random.
    \item Random + 1: Connected subgraphs with two or four nodes, combined with one additional isolated node. To avoid confounding, the isolated nodes are drawn from coastal or border regions where disconnected nodes naturally occur.
    \item Root cause: Connected subgraphs with three or five nodes in which each node has at most one parent, forming chain-like structures. This setting is useful for root-cause analysis~\citep{DBLP:conf/nips/RCD} and is fully contained within Random.
    \item Confounder: Subgraphs with four or six nodes containing a single node with multiple children (rarely observed in cases such as river splits). The multi-child node is then removed to simulate permanent hidden confounding.
    
\end{itemize}

\textbf{Root Cause Identification Datasets.} 
SWaT~\citep{DBLP:conf/cpsweek/swat} is a dataset collected from a testbed that simulates a real-world water treatment plant. It comprises data from 51 sensors in the critical infrastructure system during continuous operation, including both normal operating conditions and attack scenarios within the water treatment process.
MSDS (Multi-Source Distributed System)~\citep{DBLP:conf/esocc/msds} is developed on an OpenStack testbed and serves as a dataset for AIOps (Artificial Intelligence for IT Operations). Instances of fault injections in this system are labeled as anomalies. More detail are provided in Table~\ref{tab:dataset_rc}.

\begin{table}[h]
\centering
\caption{Details of Root Cause Identification Datasets.}
\resizebox{0.98\textwidth}{!}{
\begin{tabular}{c|c|c|c|c}
\toprule
\textbf{Dataset} & \textbf{Training Time Steps} & \textbf{Test Sequences} & \textbf{Avg. Sequence Length} & \textbf{Avg. \# of Root Causes} \\
\midrule
SWaT (51)        & 49,500                        & 20                      & 51                            & 13.35                           \\
\midrule
MSDS (10)        & 29,268                        & 4,255                   & 21                            & 3.05                            \\
\bottomrule
\end{tabular}
}
\label{tab:dataset_rc}
\end{table}

\subsection{More Details of Baselines}
We provide detailed descriptions of the baseline methods used for comparison in both Causal Discovery and Root Cause Identification tasks.

\textbf{Causal Discovery Baselines.} 
We evaluate representative methods spanning classical statistical approaches and modern deep learning paradigms. 

\begin{itemize}
\item \textbf{PCMCI}~\citep{PCMCI}: A constraint-based algorithm optimized for time-series data. It first selects preliminary parents using PC-algorithm-based selection and then estimates the causal graph using momentary conditional independence tests to control the false positive rate. It effectively handles time lags and strong autocorrelation.

\item \textbf{VarLiNGAM}~\citep{DBLP:journals/jmlr/Varlingam}: A linear non-Gaussian acyclic model extended for autoregressive processes. It assumes the data generating process is linear with non-Gaussian noise and utilizes independent component analysis to identify the causal structure and time-lagged causal effects.

\item \textbf{DYNOTEARS}~\citep{DBLP:conf/aistats/Dynotears}: A score-based method that formulates structural learning as a continuous optimization problem. It extends the NOTEARS framework to dynamic causal networks by minimizing a penalized loss function subject to acyclicity constraints, allowing for the simultaneous learning of intra-slice and inter-slice dependencies.

\item \textbf{VAR}~\citep{DBLP:conf/ijcai/VAR}: A classic statistical model that captures the linear interdependencies among multiple time series. In the context of causal discovery, Granger Causality is inferred based on the non-zero coefficients of the fitted autoregressive matrices.

\item \textbf{CDMI}~\citep{CDMI}: A method based on Causal Discovery with Mutual Information. It estimates the causal probability between time series variables by quantifying the mutual information shared between past values of potential parents and current values of the target, often utilizing non-parametric estimators to capture non-linear dependencies.

\item \textbf{TCDF}~\citep{DBLP:journals/make/tcdf}: A deep learning framework that uses attention-based CNNs. It learns a predictive model for each time series and interprets the internal attention parameters and variable importance scores to discover causal links and time delays.

\item \textbf{CUTS+}~\citep{cuts+}: An improved version of the CUTS framework designed for non-stationary time series. It employs a latent causal discovery approach where a Granger-causality-inspired graph is learned jointly with a data augmentation strategy to handle distribution shifts and improve robustness.

\item \textbf{Causal Pretraining (CP)}~\citep{DBLP:journals/corr/cp}: A self-supervised pretraining framework designed to learn causal representations from time series. It reconstructs masked sub-series to capture temporal dependencies. To evaluate the impact of the backbone architecture, we compare two variants:
\begin{itemize}
    \item \textbf{CP-Trans}: Implements the standard CP framework using a Transformer-based backbone to capture long-range dependencies via self-attention mechanisms.
    \item \textbf{CP-GRU}: Implements the CP framework using a gated recurrent unit backbone, focusing on sequential recurrent modeling.
\end{itemize}
\end{itemize}

\textbf{Root Cause Identification Baselines.} 
To evaluate the effectiveness of PTCD in identifying root causes of anomalies, we compare it with four established methods:
\begin{itemize}
\item $\epsilon$-\textbf{Diagnosis}~\citep{DBLP:conf/www/-Diagnosis}: A statistical approach that locates root causes by analyzing the correlation changes between variables. It performs pairwise significance tests (e.g., standard deviation and correlation shifts) between normal and abnormal periods to construct an anomaly propagation graph and rank potential root causes.

\item \textbf{RCD}~\citep{DBLP:conf/nips/RCD}: A framework that integrates causal structure learning with anomaly attribution. It learns a partial causal graph from data and analyzes the invariance of conditional distributions. Nodes corresponding to intervention targets are identified as root causes.

\item \textbf{CIRCA}~\citep{DBLP:conf/kdd/CIRCA}: A method that leverages domain knowledge to construct structural causal graphs. It identifies root causes by modeling the residuals of causal mechanisms; specifically, it detects nodes that exhibit significant distribution shifts in their conditional probabilities during the anomaly window.

\item \textbf{AERCA}~\citep{DBLP:conf/iclr/aerca}: A reconstruction-based method that models causal dependencies and exogenous variables using an Autoencoder framework. It posits that anomalies are generated by perturbations in exogenous factors. AERCA attributes root causes by measuring the magnitude of the reconstruction error associated with the exogenous noise terms of each variable.

\end{itemize}

\subsection{More Details of Metrics on Root Cause Identification}
\label{apx:details_metrics}
Given a multivariate time series $\mathcal{X}$, the $Recall@K$ is defined as:
\begin{equation}
    Recall@K=\frac{1}{|\mathcal{X}|} \sum_{x_i \in \mathcal{X}} \frac{\left|V_{x_i}^{(RC)} \cap\left\{R_{x_i}[k] \mid k=1,2, \ldots K\right\}\right|}{\min \left(K,\left|V_{x_i}^{(R C)}\right|\right)},
\end{equation}
where $R_{x_i}[k]$ indicates the time series at the $k$-th rank for the channel $x_i$, and $V^{(RC)}_{x_i}$ indicates a set of root cause variables over the whole channel $x_i$. Note that if a time series receives multiple exogenous interventions, it only counts as one root cause time series in $V^{(RC)}_{x_i}$. We further compute the overall performance by computing the average $Recall@K$, denoted as $Avg@K = \frac{1}{K}\sum_{k=1}^K Recall@k$.

\subsection{Implementation Details}
\label{app:exp_details}
We first pretrain PTCD using the Adam optimizer~\citep{DBLP:journals/corr/adam} with a learning rate of $10^{-3}$. We apply early stopping with a patience of 30 epochs to prevent overfitting. The loss weights $\lambda_{G}$, $\lambda_{KL}$ and $\lambda_{do}$ are set to 0.6. The model consists of 4 alternating attention blocks, and the number of Gaussian components $K$ is set to 10. For fine-tuning in both the causal discovery task and the root cause identification task, we use a learning rate of $10^{-4}$, and keep $\lambda_{KL}$ fixed at 0.6. All experiments are implemented in PyTorch and conducted on an NVIDIA A800 80GB GPU. Following AERCA~\citep{DBLP:conf/iclr/aerca}, data preprocessing is standardized across datasets using a MinMax scaler. To improve computational efficiency, we downsample the SWaT dataset every 10 seconds and the MSDS dataset every 5 time steps. The code and datasets are available at \href{}{https://anonymous.4open.science/r/PTCD-D4E3}.

\subsection{More Ablation Study}
\label{apx:ab_rc}
To investigate the impact of different components on root cause identification, we conduct an ablation study focusing on inter-window attention, intra-window attention, mixture-of-Gaussians modeling of exogenous variables, causal mixup, and intervention. As shown in Table~\ref{tab:ablation_rc}, we find that the intervention mechanism achieves superior performance in identifying root causes. This gain is attributed to the fact that interventions provide the model with information across diverse environments, enabling it to filter out spurious correlations, learn more accurate causal relationships, and enhance overall robustness. Furthermore, when the mixture-of-Gaussians constraint on exogenous variables is removed, PTCD exhibits a notable drop in performance, demonstrating the importance of instance-specific exogenous variable estimation for effective root cause identification.

\begin{table}[h]
\centering
\caption{Ablation studies on Root Cause Localization. w/o inter, w/o intra, w/o exogenous, w/o causal mixup, w/o intervention represent removing the inter-window attention, intra-window attention, exogenous, causal mixup, and intervention, respectively.}
\resizebox{0.95\textwidth}{!}{
\begin{tabular}{c|c|ccccc}
\toprule
Dataset               & Model                & Recall@1                 & Recall@3        & Recall@5                 & Recall@10                & Avg@10               \\
\midrule
\multirow{5}{*}{MSDS} & w/o inter & 0.452 & 0.792 & 0.910          & \textbf{1.000} & 0.492         \\
                      & w/o intra         & 0.498          & 0.801 & 0.910 & \textbf{1.000} & 0.835         \\
                      & w/o exogenous       & 0.227          & 0.697 & 0.893          & \textbf{1.000}          & 0.694          \\
                      & w/o causal mixup                & 0.500          & 0.792 & 0.993         & \textbf{1.000}          & 0.842          \\
                      & w/o intervention                & 0.296 & 0.798 & 0.900          & \textbf{1.000}         & 0.728          \\
                      & PTCD        &       \textbf{0.515}               &      \textbf{0.993}       &    \textbf{0.996}                 &             \textbf{1.000}         &     \textbf{0.929}     \\
\midrule
\multirow{5}{*}{SWaT} & w/o inter          & 0.200          & 0.350 & 0.350          & 0.455 &     0.351      \\
                      & w/o intra         & 0.250          & 0.375 & 0.375          & 0.450          &      0.360      \\
                      & w/o exogenous       & 0.100          & 0.150 & 0.250         & 0.350          &    0.205       \\
                      & w/o causal mixup       & 0.250 & 0.330 & 0.450 & 0.455          &   0.342    \\   
                      & w/o intervention       & 0.150          & 0.220 & 0.330          & 0.375          &    0.252      \\
                      & PTCD        &       \textbf{0.300}               &      \textbf{0.450}       &    \textbf{0.450}                  &             \textbf{0.475}         &   \textbf{0.440}  \\
\bottomrule
\end{tabular}
}
\label{tab:ablation_rc}
\end{table}

\subsection{Efficiency Analysis}
\begin{table}[h]
\centering
\caption{Efficiency comparison across different paradigms on the causalriver dataset.}
\resizebox{\textwidth}{!}{
\begin{tabular}{llccccc}
\hline
Model & Paradigm & Params (K) & Inference Time & GPU Mem (MB) & CPU Mem (MB) & \textbf{avg. AUROC} \\
\hline
TCDF & End-to-End Causal & 0.12 & 17.19 s & 2.00 & 460.45 & 0.57 \\
CUTS+ & End-to-End Causal & 8.23 & 14.19 s & 2.00 & 293.36 & 0.56 \\
\textbf{PTCD (zero-shot)} & Pre-trained Causal & 877.24 & 0.01 s & 1236.00 & 504.08 & 0.74 \\
\textbf{PTCD (fine-tune)} & Pre-trained Causal & 877.24 & 3.59 s & 1236.00 & 504.08 & \textbf{0.85} \\
\hline
\end{tabular}
}
\label{tab:eff}
\end{table}
To provide a comprehensive efficiency profile, we evaluate the parameters, inference time, and memory usage across different paradigms. As shown in Table~\ref{tab:eff}, although PTCD introduces a moderate parameter overhead due to pre-training, its fine-tuned variant achieves the best overall performance, at the expense of increased inference time compared to the zero-shot setting. This demonstrates a favorable trade-off: zero-shot PTCD offers exceptional efficiency, while the fine-tuned variant delivers superior predictive accuracy, all while retaining rapid inference capabilities.

\subsection{Parameter Sensitivity Analysis}
We conduct a hyperparameter sensitivity analysis with respect to the loss weights $\lambda_G$, $\lambda_{KL}$, and $\lambda_{do}$, which correspond to the causal graph loss, the KL-divergence loss, and the intervened loss, respectively, as well as the number of Gaussian components $K$. All experiments are performed on synthetic datasets. As shown in Figure~\ref{apx:parameter_sensitivity}, PTCD achieves optimal causal discovery performance when $\lambda_G$, $\lambda_{KL}$, and $\lambda_{do}$ are set to $0.6$. This indicates that causal graph prediction, the modeling of exogenous variables, and intervention mechanisms play a significant role in the generalization capability of the learned causal discovery. Additionally, the number of Gaussian components also influences the model's generalization performance to a certain degree.

\begin{figure}[h]
\
  \centering
  \resizebox{0.80\textwidth}{!}{%
      \begin{minipage}{\textwidth}
          \centering
      \begin{subfigure}[b]{0.48\textwidth}
        \centering
        \includegraphics[width=\linewidth]{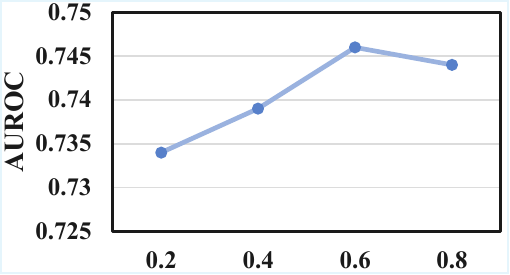}
        \caption{The causal graph loss weight $\lambda_G$.}
      \end{subfigure}\hfill
      \begin{subfigure}[b]{0.48\textwidth}
        \centering
        \includegraphics[width=\linewidth]{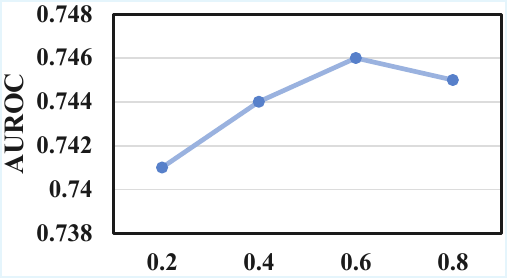}
        \caption{The KL divergence loss weight $\lambda_{KL}$.}
      \end{subfigure}\hfill
      \begin{subfigure}[b]{0.48\textwidth}
        \centering
        \includegraphics[width=\linewidth]{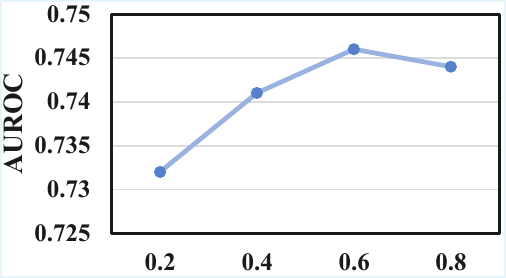}
        \caption{The intervention loss weight $\lambda_{do}$.}
      \end{subfigure}\hfill
      \begin{subfigure}[b]{0.48\textwidth}
        \centering
        \includegraphics[width=\linewidth]{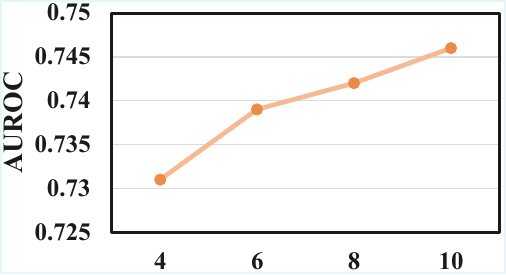}
        \caption{The number of Gaussian components $K$.}
      \end{subfigure}\hfill
      \caption{The effect of parameter $\lambda_G$, $\lambda_{KL}$, $\lambda_{do}$, and $K$ on synthetic datasets for time-series causal discovery.}
      \label{apx:parameter_sensitivity}
  \end{minipage}
  
}
\end{figure} 

\section{Limitations}
\label{app_limit}
While PTCD demonstrates strong generalization capabilities in time series causal discovery and root cause identification, it has a few limitations regarding its structural assumptions. The SCM assumes an Additive Noise Model (ANM). Although ANM is theoretically advantageous for guaranteeing causal identifiability and cleanly decoupling endogenous dynamics from exogenous anomalies for precise root cause localization, it may not capture complex, non-additive external disturbances (e.g., multiplicative noise or nonlinear entanglement) present in reality. Future work will explore flexible non-additive causal mechanisms to further enhance robustness against extreme out-of-distribution (OOD) noise.